\documentclass{article}
\usepackage{colt10e}
\usepackage{pifont}
\usepackage{times}
\usepackage{amsfonts}
\usepackage{amsmath}
\usepackage{amstext}
\usepackage[psamsfonts]{amssymb}
\usepackage{latexsym}
\usepackage{graphicx}
\usepackage{natbib}
\usepackage{algorithmic}
\usepackage{tikz}
\usepackage[vlined,algoruled,linesnumbered]{algorithm2e}
\usepackage[colorlinks,linkcolor=red,citecolor=blue,urlcolor=blue]{hyperref}

\input{Definitions}
\title{On the Finite Time Convergence of Cyclic Coordinate Descent Methods}

\author{Ankan Saha\\
  Department of Computer Science\\
  University of Chicago \\
  \texttt{\small ankans@cs.uchicago.edu}
  \And Ambuj Tewari\\
  Toyota Technological Institute\\ 
  Chicago, USA \\
  \texttt{\small tewari@ttic.edu}
}

\begin{document}

\maketitle

\begin{abstract}

Cyclic coordinate descent is a classic optimization method that has
witnessed a resurgence of interest in machine learning. Reasons for
this include its simplicity, speed and stability, as well as its
competitive performance on $\ell_1$ regularized smooth optimization
problems.  Surprisingly, very little is known about its finite time
convergence behavior on these problems. Most existing results either
just prove convergence or provide asymptotic rates. We fill this gap
in the literature by proving $O(1/k)$ convergence rates (where $k$ is
the iteration counter) for two variants of cyclic coordinate descent
under an isotonicity assumption. Our analysis proceeds by comparing
the objective values attained by the two variants with each other, as
well as with the gradient descent algorithm. We show that the iterates
generated by the cyclic coordinate descent methods remain better than
those of gradient descent uniformly over time.

\end{abstract}

\section{Introduction}
\label{sec:Introduction}

The dominant paradigm in Machine Learning currently is to cast
learning problems as optimization problems. This is clearly borne out
by approaches involving empirical risk {\em minimization}, {\em maximum}
likelihood, {\em maximum} entropy, {\em minimum} description length,
etc. As machine learning faces ever increasing and high-dimensional
datasets, we are faced with novel challenges in
designing and analyzing optimization algorithms that can adapt efficiently to
such datasets. A mini-revolution of sorts is taking place where
algorithms that were ``slow'' or ``old'' from a purely optimization
point of view are witnessing a resurgence of interest. This paper
considers one such family of algorithms, namely the {\em
coordinate descent} methods. There has been recent work demonstrating
the potential of these algorithms for solving $\ell_1$-regularized
loss minimization problems:
\begin{equation}
\label{eq:lossl1reg}
\frac{1}{n} \sum_{i=1}^n \ell(x,Z_i) + \lambda \|x\|_1
\end{equation}
where $x$ is possibly high dimensional predictor that is being learned
from the samples $Z_i = (X_i, Y_i)$ consisting of input, output pairs,
$\ell$ is a convex loss function measuring prediction performance, and
$\lambda \ge 0$ is a ``regularization'' parameter.  The use of the
$\ell_1$ norm $\|x\|_1$ (sum of absolute values of $x_i$) as a
``penalty'' or ``regularization term'' is motivated by its sparsity
promoting properties and there is a large and growing literature
studying such issues (see, e.g., \cite{Tropp06} and references therein).  In this paper,
we restrict ourselves to analyzing the behavior of coordinate descent methods on
problems like~\eqref{eq:lossl1reg} above. The general idea behind
coordinate descent is to choose, at each iteration, an index $j$ and
change $x_j$ such that objective $F$ decreases.  Choosing $j$ can be
as simple as cycling through the coordinates or a more sophisticated
coordinate selection rule can be employed.  \cite{HastTib07} use
the cyclic rule which we analyze in this paper.

Our emphasis is on obtaining {\em finite time} rates, i.e. guarantees
about accuracy of iterative optimization algorithms that hold right
from the first iteration. This is in contrast to asymptotic guarantees
that only hold once the iteration count is ``large enough'' (and
often, what is meant by ``large enough'', is left unspecified). We
feel such an emphasis is in the spirit of Learning Theory that has
distinguished itself by regarding finite sample generalization bounds
as important. For our analysis, we abstract away the particulars of
the setting above, and view~\eqref{eq:lossl1reg} as a special case of
the convex optimization problem:
\begin{equation}
\label{eq:Reg_l_1_loss}
\min_{x \in \mathbb{R}^d}\ F(x) := f(x) + \lambda \|x\|_1 \ .
\end{equation}
In order to obtain finite time convergence rates, one must assume that
$f$ is ``nice'' is some sense. This can be quantified in different ways 
including assumptions of Lipschitz continuity, differentiability or
strong convexity. We will assume that $f$ is differentiable with a
Lipschitz continuous gradient. In the context of
problem~\eqref{eq:lossl1reg}, it amounts to assuming that the loss
$\ell$ is differentiable. Many losses, such as squared loss and
logistic loss, are differentiable. Our results therefore apply to
$\ell_1$ regularized squared loss (``Lasso'') and to $\ell_1$
regularized logistic regression.

For a method as old as cyclic coordinate descent, it is surprising
that little is known about finite time convergence even under
smoothness assumptions.  As far as we know, finite time results are
not available even when $\lambda = 0$. i.e. for unconstrained smooth
convex minimization problem. Given recent empirical successes of the
method, we feel that this gap in the literature needs to be filled
urgently.  In fact, this sentiment is shared in (\cite{WuLange08}) by
the authors who lamented, ``Better understanding of the convergence
properties of the algorithms is sorely needed."  They were talking
about greedy coordinate descent methods but their comment applies to
cyclic methods as well.

The situation with gradient descent methods is much better. There are
a variety of finite time convergence results available in the
literature (\cite{Nesterov03a}). Our strategy in this paper is to
leverage these results to shed some light on the convergence of
coordinate descent methods. We do this via a series of comparison
theorems that relate variants of coordinate descent methods to each
other and to the gradient descent algorithm. To do this, we make
assumptions both on the starting point and an additional {\em
isotonicity} assumption on the gradient of the function $f$. Since
finite time $O(1/k)$ accuracy guarantees are available for gradient
descent, we are able to prove the same rates for two variants of
cyclic coordinate descent.  Here $k$ is the iteration count and the
constants hidden in the $O(\cdot)$ notation are small and known.  We
feel it should be possible to relax, or even eliminate, the additional
assumptions we make (these are detailed in section \ref{sec:Analysis}) and doing this is
an important open problem left for future work.

We find it important to state at the outset that our aim here is not
to give the best possible rates for the
problem~\eqref{eq:Reg_l_1_loss}. For example, even among
gradient-based methods, faster $O(1/k^2)$ finite time accuracy bounds
can be achieved using Nesterov's celebrated 1983 method
(\cite{Nesterov83}) or its later variants. Instead, our goal is to
better understand cyclic coordinate descent methods and their
relationship to gradient descent.

\paragraph{Related Work}
Coordinate descent methods are quite old and we cannot attempt a
survey here. Instead, we refer the reader to \cite{Tseng01} and \cite{TseYun09a}
that summarize previous work and also present analyses for
coordinate descent methods. These consider cyclic coordinate descent as well as
versions that use more sophisticated coordinate selection rules. However,
as mentioned above, the analyses either establish convergence without
rates or give asymptotic rates that hold after sufficiently many
iterations have occurred. An exception is \cite{TseYun09} that
does give finite time rates but for a version of coordinate descent
that is not cyclic. Finite time guarantees for a greedy version
(choosing $j$ to be the coordinate of the current gradient with the
maximum value) also appear in \cite{Clarkson08}. The author essentially
considers minimizing a smooth convex function over the probability
simplex and also surveys previous work on greedy coordinate descent in
that setting. For finite time (expected) accuracy bounds for
stochastic coordinate descent (choose $j$ uniformly at random) for
$\ell_1$ regularization, see \cite{ShaiAmbuj09}.

We mentioned that the empirical success reported in \cite{HastTib07} was our
motivation to consider cyclic coordinate descent for $\ell_1$
regularized problems. They consider the Lasso problem:
\begin{equation}
\label{eq:lasso}
\min_{x \in \mathbb{R}^d}\ \frac{1}{2n} \| \mathbf{X}x - Y \|^2 + \lambda \|x\|_1\ ,
\end{equation}
where $\mathbf{X} \in \mathbb{R}^{n \times d}$ and
$Y \in \mathbb{R}^n$. In this case, the smooth part $f$ is a quadratic 
\begin{align}
        \label{eq:quadratic}
        f(x) = \tfrac{1}{2}\inner{Ax}{x} + \inner{b}{x}
\end{align} 
where $A = \mathbf{X}^\top\mathbf{X}$ and $b = -\mathbf{X}^\top Y$.  Note that
$A$ is symmetric and positive semidefinite. 
Cyclic coordinate descent has also been applied to the $\ell_1$-regularized logistic regression
problem~\citep{GenLewMad07}. Since the logistic loss is differentiable, this problem also falls
into the framework of this paper.

\paragraph{Outline}
Notation and necessary definitions are given in
section \ref{sec:Preliminaries}. The gradient descent algorithm along
with two variants of cyclic coordinate descent are presented in
section \ref{sec:Algorithms}. Section \ref{sec:Analysis} spells out
the additional assumptions on $f$ that our current analysis needs. It
also proves results comparing the iterates generated by the three
algorithms considered in the paper when they are all started from the
same point. Similar comparison theorems in the context of solving
a system of non-linear equations using Jacobi and Gauss-Seidel methods
appear in \cite{Rheinboldt70}.
The results in section \ref{sec:Analysis} set the stage for the main results given in
section \ref{sec:Rates}. This section converts the comparison between
iterates into a comparison between objective function values achieved
by the iterates. The finite time convergence rates of cyclic
coordinate descent are then inferred from rates for gradient descent.
There are plenty of issues that are still
unresolved. Section \ref{sec:Conclusion} discusses some of them and
provides a conclusion.

\section{Preliminaries and Notation}
\label{sec:Preliminaries}

We use the lowercase letters $x$, $y$, $z$, $g$ and $\gamma$ to refer 
to vectors throughout the paper. Normally parenthesized superscripts, 
like $x^{(k)}$ refer to vectors as well, whereas subscripts refer to 
the components of the corresponding vectors. For any positive integer 
$k$, $[k] := \{1,\ldots,k\}$. $\sgn(a)$ is the interval-valued sign 
function, i.e. $\sgn(a) = \{1\}$ or $\{-1\}$ corresponding to $a>0$ or $a<0$. For 
$a = 0$, $\sgn(a) = [-1,1]$.
Unless otherwise specified, $\| \cdot \|$ refers to the Euclidean norm
$\nbr{x} := \left(\sum_{i}x_{i}^{2}\right)^ {\frac{1}{2}}$,
$\|\cdot\|_1$ will denote the $l_1$ norm, $\nbr{x}_1 =
\left(\sum_i|x_i| \right)$, $\inner{\cdot}{\cdot}$ denotes the
Euclidean dot product $\inner{x}{y} = \sum_{i} x_{i}y_{i}$. Through out
the paper inequalities between vectors are to be interpreted component
wise \ie\ $x \ge y$ means that $x_i \ge y_i$ for all $i \in [d]$.
The following definition will be used extensively
in the paper:
\begin{definition}
  \label{def:lip-cont-grad}
  Suppose a function $f: \RR^d \to \RR$ is differentiable on $\RR^d$. 
  Then $f$ is said to have Lipschitz continuous gradient (\lcg) with
  respect to a norm $\|\cdot\|$ if there exists a constant $L$ such that
  \begin{align}
    \label{eq:lip-cont-grad}
    \| \nabla f(x) - \nabla f(x')\| \leq L \| x - x'\| \qquad
    \forall\ x, x'\in \RR^d.
  \end{align}
\end{definition}

An important fact (see, e.g., \cite[Thm. 2.1.5]{Nesterov03a}) we will
use is that if a function $f$ has Lipschitz continuous gradient with
respect to a norm $\|\cdot\|$, then it satisfies the following
generalized bounded Hessian property
  \begin{align}
    \label{eq:generalized_hessian}
    f(x) \leq f(x') + \inner{\grad f(x')}{x-x'} + \frac{L}{2}\|x-x'\|^2.
  \end{align}

An operator $T:\mathbb{R}^d \to \mathbb{R}$ is said to be {\em isotone} iff
\begin{equation}
\label{eq:isotone}
x \ge y \quad \Rightarrow \quad T(x) \ge T(y).
\end{equation}

An important isotone operator that we will frequently deal with is the
{\em shrinkage} operator $\Sbb_\tau:\mathbb{R}^d \to \mathbb{R}$
defined, for $\tau > 0$, as
\begin{align}
\label{eq:shrinkage}
[\Sbb_\tau(x)]_i := S_\tau(x_i)
\end{align}
where $S_\tau(a)$ is the scalar shrinkage operator:
\begin{equation}
\label{eq:scshrinkage}
S_\tau(a) :=
\begin{cases}
a - \tau & a > \tau \\
0 & a \in [-\tau,\tau] \\
a + \tau & a < -\tau.
\end{cases}
\end{equation}


\section{Algorithms}
\label{sec:Algorithms}

We will consider three iterative algorithms for solving the
minimization problem~\eqref{eq:Reg_l_1_loss}. All of them enjoy the
descent property: $F(x^{(k+1)}) \le F(x^{(k)})$ for successive iterates
$x^{(k)}$ and $x^{(k+1)}$.

\begin{algorithm}
\begin{algorithmic}
  \STATE Initialize: Choose an appropriate initial point $x^{(0)}$. 
  \FOR{$k=0,1,\ldots$}
  \STATE $x^{(k+1)} \leftarrow \Sbb_{\lambda/L}(x^{(k)} - \frac{\nabla f(x^{(k)})}{L})$
  \ENDFOR
\end{algorithmic}
\caption{Gradient Descent (\GD)}
\label{alg:gd}
\end{algorithm}

Algorithm~\ref{alg:gd}, known as Gradient Descent (\GD), is one of the
most common iterative algorithms used for convex optimization (See
\cite{BeckTeb09}, \cite{DucSing09} and references therein). It is
based on the idea that using corollary \eqref{eq:generalized_hessian}
to generate a linear approximation of $f$ at the current iterate
$x^{(k)}$, we can come up with the following global upper
approximation of $F$:
\[
	F(x) \leq f(x^{(k)}) + \inner{ \nabla f(x^{(k)}) }{ x -
          x^{(k)} } + \frac{L}{2}\| x - x^{(k)} \|^2 + \lambda \| x
        \|_1 \ .
\]
It is easy to show that the above approximation is minimized at $x =
\Sbb_{\lambda/L}(x^{(k)} - \nabla f(x^{(k)})/L)$ (\cite{BeckTeb09}). 
This is the next iterate for the \GD\ algorithm. We call it 
``Gradient Descent'' as it reduces to the following algorithm
\[
	x^{(k+1)} = x^{(k)} - \frac{\nabla f(x^{(k)})}{L}
\]
when there is no regularization (i.e. $\lambda = 0$). Finite time
convergence rate for the \GD\ algorithm are well known.

\begin{theorem}
\label{thm:gd}
Let $\cbr{ x^{(k)} }$ be a sequence generated by the \GD\ 
algorithm. Then, for any minimizer $x^\star$
of~\eqref{eq:Reg_l_1_loss}, and $\forall k \ge 1$,
\[
	F(x^{(k)}) - F(x^\star) \le \frac{ L \| x^\star - x^{(0)} \|^2}{2\,k}
\] 
\end{theorem}
The above theorem can be found in, e.g., \cite[Thm. 3.1]{BeckTeb09}.

\begin{algorithm}
\begin{algorithmic}
  \STATE Initialize: Choose an appropriate initial point $y^{(0)}$.
  \FOR{$k=0,1,\ldots$} 
  \STATE $y^{(k,0)} \leftarrow y^{(k)}$ 
  \FOR{$j=1$ to $d$}
  \STATE $y^{(k,j)}_j \leftarrow S_{\lambda/L}(y^{(k,j-1)}_j - [\nabla
    f(y^{(k,j-1)})]_j \,/\, L)$ 
  \STATE $\forall i\neq j$, $y^{(k,j)}_i \leftarrow y^{(k,j-1)}_i$ 
  \ENDFOR 
  \STATE $y^{(k+1)} \leftarrow y^{(k,d)}$ 
  \ENDFOR
\end{algorithmic}
\caption{Cyclic Coordinate Descent (\CCD)}
\label{alg:ccd}
\end{algorithm}

The second algorithm, Cyclic Coordinate Descent (\CCD), instead of
using the current gradient to update all components simultaneously,
goes through them in a cyclic fashion. The next ``outer'' iterate
$y^{(k+1)}$ is obtained from $y^{(k)}$ by creating a series of $d$
intermediate or ``inner'' iterates $y^{(k,j)}$, $j \in [d]$, where
$y^{(k,j)}$ differs from $y^{(k,j-1)}$ only in the $j$th coordinate
whose value can be found by minimizing the following one-dimensional 
over-approximation of $F$ over the scalar $\alpha$:
\begin{equation}
\label{eq:1d_over_approx}
f(y^{(k,j-1)}) + \lambda \sum_{i \neq j} |y^{(k,j-1)}_i| + [
  \nabla f(y^{(k,j-1)}) ]_j \cdot (\alpha - y^{(k,j-1)}_j) +
\frac{L}{2} (\alpha - y^{(k,j-1)})_j^2 + \lambda |\alpha|\ .
\end{equation}
It can again be verified that the above minimization has the closed
form solution 
\begin{align*}
  \alpha = S_{\lambda/L}\rbr{y^{(k,j-1)}_j - 
    \frac{[\grad f(y^{(k,j-1)})]_j}{L}}
\end{align*}
which is what \CCD\ chooses $y^{(k,j)}_j$ to be. Once all coordinates
have been cycled through, $y^{(k+1)}$ is simply set to be
$y^{(k,d)}$. Let us point out that in an actual implementation, the
inner iterates $y^{(k,j)}$ would not be computed separately but
$y^{(k)}$ would be updated ``in place''. For analysis purposes, it is
convenient to give names to the intermediate iterates. Note that for all 
$j \in \cbr{0,1,\hdots,d}$, the inner iterate looks like
\[
y^{(k,j)} = \sbr{y^{(k+1)}_1, \hdots,y^{(k+1)}_j,y^{(k)}_{j+1},\hdots, y^{(k)}_d}\ .
\]

In the \CCD\ algorithm updating the $j$th coordinate uses the newer
gradient value $\grad f(y^{(k,j-1)})$ rather than $\grad f(y^{(k)})$
which is used in \GD. This makes \CCD\ inherently sequential. In
contrast, different coordinate updates in \GD\ can easily be done by
different processors in parallel. However, on a single processor, we
might hope \CCD\ converges faster than \GD\ due to the use of ``fresh''
information. Therefore, it is natural to expect that \CCD\ should enjoy the
finite time convergence rate given in Theorem~\ref{thm:gd} ( or
better). We show this is indeed the case under an {\em isotonicity
  assumption} stated in Section~\ref{sec:Analysis} below. Under the
assumption, we are actually able to show the correctness of the
intuition that \CCD\ should converge faster than \GD.
\begin{algorithm}
\begin{algorithmic}
\STATE Initialize: Choose an appropriate initial point $z^{(0)}$.
\FOR{$k=0,1,\ldots$} 
\STATE $z^{(k,0)} \leftarrow z^{(k)}$ 
\FOR{$j=1$ to $d$}
\STATE $z^{(k,j)}_j \leftarrow \argmin_{\alpha}\ F(z^{(k,j-1)}_1,\ldots,
z^{(k,j-1)}_{j-1},\alpha,z^{(k,j-1)}_{j+1},\ldots,z^{(k,j-1)}_d)$
\STATE $\forall i\neq j$, $z^{(k,j)}_i \leftarrow z^{(k,j-1)}_i$ 
\ENDFOR 
\STATE $z^{(k+1)} \leftarrow z^{(k,d)}$ 
\ENDFOR
\end{algorithmic}
\caption{Cyclic Coordinate Minimization}
\label{alg:ccm}
\end{algorithm}

The third and final algorithm that we consider is Cyclic Coordinate
Minimization (\CCM). The only way it differs from \CCD\ is that instead
of minimizing the one-dimensional
over-approximation~\eqref{eq:1d_over_approx}, it chooses $z^{(k,j)}_j$
to minimize,
\[
F(z^{(k,j-1)}_1,\ldots,z^{(k,j-1)}_{j-1},\alpha,z^{(k,j-1)}
_{j+1},\ldots,z^{(k,j-1)}_d)
\]
over $\alpha$. In a sense, \CCM\ is not actually an algorithm as it does 
not specify how to minimize $F$ for any arbitrary smooth function $f$. 
An important case when the minimum can be computed exactly is when $f$ is
quadratic as in \eqref{eq:quadratic}. In that case, we have
\[
	z^{(k,j)}_j = S_{\lambda/A_{j,j}}
	\left(
		z^{(k,j-1)}_j - \frac{[Az^{(k,j-1)} + b]_j}{A_{j,j}}
	\right)\ .
\] 
If there is no closed form solution, then we might have to
resort to numerical minimization in order to implement \CCM. This is
usually not a problem since one-dimensional convex functions can be
minimized numerically to an extremely high degree of accuracy in a few
steps. For the purpose of analysis, we will assume that
an exact minimum is found. Again, intuition suggests that the accuracy
of \CCM\ after any fixed number of iterations should be better than that
of \CCD\ since \CCD\ only minimizes an over-approximation. Under the same
isotonicity assumption that we mentioned above, we can show that this
intuition is indeed correct.

We end this section with a cautionary remark regarding terminology. In
the literature, \CCM\ appears much more frequently than \CCD\ and it is
actually the former that is often referred to as ``Cyclic Coordinate
Descent'' (See \cite{HastTib07} and references therein). Our reasons
for considering \CCD\ are: (i) it is a nice, efficient alternative to
\CCM, and (ii) a stochastic version of \CCD (where the coordinate to
update is chosen randomly and not cyclically) is already known to
enjoy finite time $O(1/k)$ expected convergence rate
(\cite{ShaiAmbuj09}).

\section{Analysis}
\label{sec:Analysis}

We already mentioned the known convergence rate for \GD\ 
(Theorem~\ref{thm:gd}) above. Before delving into the analysis, it is
necessary to state an assumption on $f$ which accompanied by
appropriate starting conditions results in particularly interesting
properties of the convergence behavior of \GD, as described in lemma
\ref{lem:GD_compare}. The \GD\ algorithm generates iterates by applying
the operator
\begin{align}
  \label{eq:GD_operator}
  T_{\GD}(x) := \Sbb_{\lambda/L}\left( x - \frac{\nabla f(x) }{ L} \right)
\end{align}
repeatedly. It turns out that if $T_{\GD}$ is an isotone operator then
the \GD\ iterates satisfy lemma \ref{lem:GD_compare} which is
essential for our convergence analysis. The above operator is a
composition of $\Sbb_{\lambda/L}$, an isotone operator, and $\Ib -
\nabla f/L$ (where $\Ib$ denotes the identity operator). To ensure
overall isotonicity, it suffices to assume that $\Ib - \nabla f/L$ is
isotone. This is formally stated as:
\begin{assumption}
The operator
$
	x \mapsto x - \frac{\nabla f(x)}{L}
$
is isotone.
\end{assumption}

Similar assumptions appear in the literature comparing Jacobi and Gauss-Seidel methods
for solving linear equations~\cite[Chap. 2]{BertTsit89}. When the function $f$ is
quadratic as in~\eqref{eq:quadratic}, our assumption is equivalent to assuming that the off-diagonal
entries in $A$ are non-positive, i.e. $A_{i,j} \le 0$ for all $i\neq j$. For a general smooth $f$, the
following condition is sufficient to make the assumption true: $f$ is twice-differentiable and the Hessian
$\nabla^2f(x)$ at any point $x$ has non-positive off-diagonal entries.

In the next few subsections, we will see how the isotonicity
assumption leads to an isotonically decreasing (or increasing)
behavior of \GD, \CCD\ and \CCM\ iterates under appropriate starting
conditions. To specify what these starting conditions are, we need the
notions of super- and subsolutions.

\begin{definition}
A vector $x$ is a supersolution iff
$
	x \ge \Sbb_{\lambda}\left( x - \nabla f(x) \right)
$.
Analogously, $x$ is a subsolution iff
$
	x \le \Sbb_{\lambda}\left( x - \nabla f(x) \right)
$. 
\end{definition}

Since the inequalities above are vector inequalities, an arbitrary $x$
may neither be a supersolution nor a subsolution. The names
``supersolution'' and ``subsolution'' are justified because equality
holds in the definitions above, \ie\
$
	x = \Sbb_{\lambda}\left( x - \nabla f(x) \right)
$
iff $x$ is a minimizer of $F$. To see this, note that subgradient optimality conditions
say that $x$ is a minimizer of $F = f + \lambda \| \cdot \|_1$ iff 
for all $j \in [d]$ 
\begin{align}
  \label{eq:optimality_cond}
  0 \in [\nabla f(x)]_j + \lambda \sgn(x_j)\ .
\end{align}
Further, it is easy to see that, 
\begin{align}
  \label{eq:equivalence}
\forall a,b \in \mathbb{R},\ \tau >0, \qquad
  0 \in b + \lambda \sgn(a) \qquad \Leftrightarrow \qquad a =
  S_{\lambda/\tau}(a - b/\tau)
\end{align}
We prove a couple of properties of super- and subsolutions that will
prove useful later. The first property refers to the scale invariance
of the definition of super- and subsolutions
and the second property is the monotonicity of a single variable function.
\begin{lemma}
\label{lem:scale}
If for any $\tau > 0$,
\begin{align}
  \label{eq:scale}
	x \ge \Sbb_{\lambda/\tau}\left( x - \frac{\nabla f(x)}{\tau} \right)
\end{align}
then $x$ is a supersolution. If $x$ is a supersolution then the above
inequality holds for all $\tau > 0$.

Similarly, if for any $\tau > 0$,
\[
	x \le \Sbb_{\lambda/\tau}\left( x - \frac{\nabla f(x)}{\tau}
        \right)
\]
then $x$ is a subsolution. If $x$ is a subsolution then the above
inequality holds for all $\tau > 0$.
\end{lemma}
\begin{proof}
  See Appendix \ref{sec:scale_proof}
\end{proof}

\begin{lemma}
\label{lem:monotonic}
If $x$ is a supersolution (resp. subsolution) then for any $j$, the
function
\[
	\tau \mapsto S_{\lambda/\tau}\left( x_j - \frac{[\nabla
            f(x)]_j}{\tau} \right)
\]
is monotonically nondecreasing (resp. nonincreasing).
\end{lemma}
\begin{proof}
  See Appendix \ref{sec:monotonic_proof}
\end{proof}

\subsection{Gradient Descent}
\label{subsec:GD}

\begin{lemma}
  \label{lem:GD_compare}
  If $x^{(0)}$ is a supersolution and $\cbr{x^{(k)}}$ is the sequence
  of iterates generated by the \GD\ algorithm then $\forall k \geq 0$,
  \begin{align*}
    1)\quad x^{(k+1)} &\leq x^{(k)}
    &
    2)\quad x^{(k)} \text{ is a supersolution}
  \end{align*}
  If $x^{(0)}$ is a subsolution and $\cbr{x^{(k)}}$ is the sequence
  of iterates generated by the \GD\ algorithm then $\forall k \geq 0$,
  \begin{align*}
    1)\quad x^{(k+1)} &\geq x^{(k)}
    &
    2) \quad x^{(k)} \text{ is a subsolution}
  \end{align*}
\end{lemma}

\begin{proof}
  We only prove the supersolution case. The proof for the subsolution
  case is analogous. We start with a supersolution $x^{(0)}$.
  Consider the operator
  \begin{align*}
    T_{\GD}(x) := \Sbb_{\lambda/L}\left(x - \frac{\grad f(x)}{L}\right)
  \end{align*}
  given by \eqref{eq:GD_operator}. By the isotonicity assumption,
  $T_{\GD}$ is an isotone operator. We will prove by induction that
  $T_{\GD}(x^{(k)}) \le x^{(k)}$. This proves that $x^{(k+1)} \le
  x^{(k)}$ since $x^{(k+1)} = T_{\GD}(x^{(k)})$. Using lemma
  \ref{lem:scale}, the second claim follows by the
  definition of the $T_{\GD}$ operator.

  The base case $T_{\GD}(x^{(0)}) \le x^{(0)}$ is true by
  Lemma~\ref{lem:scale} since $x^{(0)}$ is given to be a
  supersolution. Now assume $T_{\GD}(x^{(k)}) \le x^{(k)}$. Applying
  the isotone operator $T_{\GD}$ on both sides we get
  $T_{\GD}(T_{\GD}(x^{(k)})) \le T_{\GD}(x^{(k)})$. This is the same
  as $T_{\GD}(x^{(k+1)}) \le x^{(k+1)}$ by definition of $x^{(k+1)}$ which completes 
  our inductive claim.

\end{proof}

\subsection{Cyclic Coordinate Descent (\CCD)}
\label{subsec:CCD}

\begin{lemma}
  \label{lem:CCD_Compare}
  If $y^{(0)}$ is a supersolution and $\cbr{y^{(k)}}$ is the sequence
  of iterates generated by the \CCD\ algorithm then $\forall k \geq 0$,
  \begin{align*}
  1) \quad y^{(k+1)} &\leq y^{(k)}
  &
  2) \quad y^{(k)} \text{ is a supersolution}
  \end{align*}
  If $y_0$ is a subsolution and $\cbr{y^{(k)}}$ is the sequence
  of iterates generated by the \CCD\ algorithm then $\forall k \geq 0$,
  \begin{align*}
  1)\quad y^{(k+1)} &\geq y^{(k)}
  &
  2)\quad y^{(k)} \text{ is a subsolution}
  \end{align*}
\end{lemma} 

\begin{proof}
  We will only prove the supersolution case as the subsolution proof
  is analogous. We start with a supersolution $y^{(0)}$. We
  will prove the following: If $y^{(k)}$ is a supersolution then,
  \begin{equation}
  \label{eq:claim1}
  y^{(k+1)} \le y^{(k)} \ ,
  \end{equation}
  \begin{equation}
  \label{eq:claim2}
  y^{(k+1)} \text{ is a supersolution}
  \end{equation}
  Then the lemma follows by induction on $k$. Let us make the
  induction assumption that $y^{(k)}$ is a supersolution and try to
  prove~\eqref{eq:claim1} and~\eqref{eq:claim2}. To prove these, we
  will show that $y^{(k,j)} \le y^{(k)}$ and $y^{(k,j)}$ is a
  supersolution by induction on $j \in \cbr{0,1,\hdots,d}$. This
  proves \eqref{eq:claim1} and \eqref{eq:claim2} for $y^{(k+1)}$ since
  $y^{(k+1)} = y^{(k,d)}$. 

  For the base case ($j=0$) of the
  induction, note that $y^{(k,0)} \le y^{(k)}$ is trivial since the
  two vectors are equal. For the same reason, $y^{(k,0)}$ is a
  supersolution since we have assumed $y^{(k)}$ to be a supersolution.
  Now assume $y^{(k,j-1)} \le y^{(k)}$ and $y^{(k,j-1)}$ is a
  supersolution for some $j > 0$. We want to show that $y^{(k,j)} \le
  y^{(k)}$ and $y^{(k,j)}$ is a supersolution.

  Since $y^{(k,j-1)}$ and $y^{(k,j)}$ differ only in the $j$th
  coordinate, to show that $y^{(k,j)} \le y^{(k)}$ given $y^{(k,j-1)}
  \le y^{(k)}$, it suffices to show that $y^{(k,j)} \le y^{(k,j-1)}$, i.e.
  \begin{equation}
  \label{eq:jthentry}
	y^{(k,j)}_j \le y^{(k,j-1)}_j = y^{(k)}_j \ .
  \end{equation}
  Since $y^{(k,j-1)} \le y^{(k)}$ applying the isotone operator $\Ib -
  \nabla f/L$ on both sides and taking the $j$th coordinate gives,
  \[
   	y^{(k,j-1)}_j - \frac{ [\nabla f(y^{(k,j-1)})]_j }{L} \le
        y^{(k)}_j - \frac{ [\nabla f(y^{(k)})]_j }{L}
  \]
  Applying the scalar shrinkage operator on both sides gives,
  \begin{align*}
    S_{\lambda/L}\left(y^{(k,j-1)}_j - \frac{ [\nabla
        f(y^{(k,j-1)})]_j }{L}\right) &\le
    S_{\lambda/L}\left(y^{(k)}_j - \frac{ [\nabla f(y^{(k)})]_j
    }{L}\right)
    \le y^{(k)}_j
  \end{align*}
  The left hand side is $y^{(k,j)}_j$ by definition while the second
  inequality follows because $y^{(k)}$ is a supersolution.  Thus, we
  have proved~\eqref{eq:jthentry}.
  
  Now we prove that $y^{(k,j)}$ is a supersolution. Note that we have already shown $y^{(k,j)} \le
  y^{(k,j-1)}$. Applying the isotone operator $\Ib - \frac{\nabla
    f}{L}$ on both sides gives,
  \begin{gather}
  \label{eq:forj}
  y^{(k,j)}_j - \frac{ [ \nabla f(y^{(k,j)}) ]_j}{L} \le y^{(k,j-1)}_j
  - \frac{ [ \nabla f(y^{(k,j-1)}) ]_j}{L} \ , \\
  \label{eq:forothers}
  \forall i\neq j,\ y^{(k,j)}_i - \frac{ [ \nabla f(y^{(k,j)}) ]_i}{L}
  \le y^{(k,j-1)}_i - \frac{ [ \nabla f(y^{(k,j-1)}) ]_i}{L} \ .
  \end{gather}
  Applying a scalar shrinkage on both sides of~\eqref{eq:forj} gives,
  \[
  	S_{\lambda/L}\left(y^{(k,j)}_j - \frac{ [ \nabla f(y^{(k,j)})
          ]_j}{L}\right) \le S_{\lambda/L}\left(y^{(k,j-1)}_j - \frac{
          [ \nabla f(y^{(k,j-1)}) ]_j}{L} \right)\ .
  \]
  Since the right hand side is $y^{(k,j)}_j$ by definition, we have,
  \begin{equation}
    \label{eq:superforj}
    S_{\lambda/L}\left(y^{(k,j)}_j - \frac{ [ \nabla f(y^{(k,j)})
      ]_j}{L}\right) \le y^{(k,j)}_j \ .
  \end{equation}
  For $i \neq j$, we have
  \begin{align}
  \notag
  y^{(k,j)}_i = y^{(k,j-1)}_i
  &\ge S_{\lambda/L}\left( y^{(k,j-1)}_i - \frac{[ \grad f(y^{(k,j-1)})]_i }{L} \right)\\
  \label{eq:superforothers}
  &\ge S_{\lambda/L}\left( y^{(k,j)}_i - \frac{[ \grad f(y^{(k,j)})]_i }{L} \right) \ .
  \end{align}
  The first inequality above is true because $y^{(k,j-1)}$ is a
  supersolution (by Induction Assumption) (and Lemma~\ref{lem:scale}).
  The second follows from~\eqref{eq:forothers} by applying a scalar
  shrinkage on both sides. Combining~\eqref{eq:superforj}
  and~\eqref{eq:superforothers}, we get
  \begin{align*}
    y^{(k,j)} \geq \Sbb_{\lambda/L}\left(y^{(k,j)} - \frac{\grad f(y^{(k,j)})}{L} \right)  
  \end{align*}
  which proves, using Lemma~\ref{lem:scale}, that $y^{(k,j)}$ is a supersolution.
  
\end{proof}

\subsection{Comparison: \GD\ vs. \CCD}
\label{subsec:Compare_GD_CCD}

\begin{theorem}
\label{thm:GD_CCD_Compare}
Suppose $\cbr{ x^{(k)} }$ and $\cbr{ y^{(k)} }$ are the sequences of
iterates generated by the \GD\ and \CCD\ algorithms respectively when
started from the same supersolution $x^{(0)} = y^{(0)}$. Then,
$\forall k\ge 0$,
\[
	y^{(k)} \le x^{(k)}\ .
\]
On the other hand, if they are started from the same subsolution
$x^{(0)} = y^{(0)}$ then the sequences satisfy, $\forall k \ge 0$,
\[
	y^{(k)} \ge x^{(k)}\ .
\]
\end{theorem}
\begin{proof}
We will prove lemma \ref{thm:GD_CCD_Compare} only for the
supersolution case by induction on $k$. The base case is trivial since
$y^{(0)} = x^{(0)}$. Now assume $y^{(k)} \le x^{(k)}$ and we will
prove $y^{(k+1)} \le x^{(k+1)}$. Fix a $j \in [d]$. Note that we have,
\[
	y^{(k+1)}_j = y^{(k,j)}_j = S_{\lambda/L}\left(y^{(k,j-1)}_j -
        \frac{[\nabla f(y^{(k,j-1)})]_j }{ L}\right)\ .
\]
By Lemma~\ref{lem:CCD_Compare}, $y^{(k,j-1)} \le y^{(k)}$. Applying
the isotone operator $S_{\lambda/L} \circ (\Ib - \nabla f/L)$ on both
sides and taking the $j$th coordinate gives,
\[
	S_{\lambda/L}\rbr{ y^{(k,j-1)}_j - \frac{ [\grad f(y^{(k,j-1)})]_j }{L} }
	\leq
	S_{\lambda/L}\rbr{ y^{(k)}_j     - \frac{ [\grad f(y^{(k)})    ]_j }{L} } \ .
\]
Combining this with the previous equation gives,
\begin{equation}
\label{eq:ybound}
	y^{(k+1)}_j \le S_{\lambda/L}\rbr{ y^{(k)}_j  - \frac{ [\grad f(y^{(k)}) ]_j }{L} } \ .
\end{equation}
Since $y^{(k)} \le x^{(k)}$ by induction hypothesis, applying the
isotone operator $S_{\lambda/L} \circ (\Ib - \nabla f/L)$ on both
sides and taking the $j$th coordinate gives,
\[
	S_{\lambda/L}\rbr{ y^{(k)}_j  - \frac{ [\grad f(y^{(k)}) ]_j }{L} }
	\le
	S_{\lambda/L}\rbr{ x^{(k)}_j  - \frac{ [\grad f(x^{(k)}) ]_j }{L} }\ .
\]
By definition,
\begin{equation}
\label{eq:vdef}
	x^{(k+1)}_j = S_{\lambda/L}\rbr{ x^{(k)}_j  - \frac{ [\grad f(x^{(k)}) ]_j }{L} } \ .
\end{equation}
Combining this with the previous inequality and ~\eqref{eq:ybound} gives,
\[
	y^{(k+1)}_j \le x^{(k+1)}_j\ .
\]
Since $j$ was arbitrary this means $y^{(k+1)} \le x^{(k+1)}$ and the
proof is complete.
\end{proof}

\subsection{Cyclic Coordinate Minimization (\CCM)}
\label{subsec:CCM}

Since \CCM\ minimizes a one-dimensional restriction of the function
$F$, let us define some notation for this subsection.  Let,
\begin{align*}
\fres{j}(\alpha;x) &:= f(x_1,\ldots,x_{j-1},\alpha,x_{j+1},\ldots,x_d) \\
\Fres{j}(\alpha;x) &:= F(x_1,\ldots,x_{j-1},\alpha,x_{j+1},\ldots,x_d) \ .
\end{align*}
With this notation, \CCM\ update can be written as:
\begin{align}
\label{eq:ccmup1}
z^{(k,j)}_j &= \argmin_{\alpha}\ \Fres{j}(\alpha; z^{(k,j-1)}) \\
\notag
\forall i\neq j,\ z^{(k,j)}_i &= z^{(k,j-1)}_i \ .
\end{align}
In order to avoid dealing with infinities in our analysis, we
want to ensure that the minimum in~\eqref{eq:ccmup1} above is attained
at a finite real number. This leads to the following assumption. 

\begin{assumption}
\label{asmp:strict}
For any $x \in \mathbb{R}^d$ and any $j \in [d]$, the one-variable
function $\fres{j}(\alpha;x)$ (and hence $\Fres{j}(\alpha;x)$) is
strictly convex.
\end{assumption}

This is a pretty mild assumption: considerably weaker than assuming,
for instance, that the function $f$ itself is strictly convex. For
example, when $f$ is quadratic as in \eqref{eq:quadratic}, then the
above assumption is equivalent to saying that the diagonal entries
$A_{j,j}$ of the positive semi definite matrix $A$ are all strictly
positive. This is much weaker than saying that $f$ is strictly convex
(which would mean $A$ is invertible).

The next lemma shows that the \CCM\ update can be represented in a way
that makes it quite similar to the \CCD\ update.

\begin{lemma}
\label{lem:CCM_update}
Fix $k \ge 0, j \in [d]$ and consider the
\CCM\ update~\eqref{eq:ccmup1}. Let $g(\alpha) = \fres{j}(\alpha;
z^{(k,j-1)})$. If the update is non-trivial, i.e. $z^{(k,j)}_j \neq
z^{(k,j-1)}_j$, it can be written as
\[
	z^{(k,j)}_j = S_{\lambda/\tau}\left( z^{(k-1,j)}_j - \frac{
          \sbr{\grad f(z^{(k,j-1)})}_j }{\tau} \right)
\]
for
\begin{equation}
\label{eq:diffratio}
	\tau = \frac{g'(z^{(k,j)}_j) - g'(z^{(k,j-1)}_j)}{ z^{(k,j)}_j
          - z^{(k,j-1))}_j }\ .
\end{equation}
Furthermore, we have $0 < \tau \le L$.
\end{lemma}
\begin{proof} 
  See Appendix \ref{sec:CCM_Update}
\end{proof}

We point out that this lemma is useful only for the analysis of \CCM\
and not for its implementation (as $\tau$ depends recursively on
$z^{(k,j)}_j$) except in an important special case. In the quadratic
example \eqref{eq:quadratic}, $g(\alpha)$ is a one-dimensional
quadratic function. In this case $\tau$ does not depend on
$z^{(k,j)}_j$ and is simply $A_{j,j}$. This leads to an efficient
implementation of \CCM\ for quadratic $f$.

We are now equipped with everything to prove the following behavior of
the \CCM\ iterates.

\begin{lemma}
\label{lem:CCM_Compare}
  If $z_0$ is a supersolution and $\cbr{z^{(k)}}$ is the sequence of
  iterates generated by the \CCM\ algorithm then $\forall k \geq 0$,
  \begin{align*}
  1)\quad z^{(k+1)} &\leq z^{(k)}
  &
  2)\quad z^{(k)} \text{ is a supersolution}
  \end{align*}
  If $z_0$ is a subsolution and $\cbr{z^{(k)}}$ is the sequence of
  iterates generated by the \CCD\ algorithm then $\forall k \geq 0$,
  \begin{align*}
  1)\quad z^{(k+1)} &\geq z^{(k)}
  &
  2)\quad z^{(k)} \text{ is a subsolution}
  \end{align*}
\end{lemma}
\begin{proof}
  Again, we will only prove the supersolution case as the subsolution
  case is analogous. We are given that $z^{(0)}$ is a
  supersolution. We will prove the following: if $z^{(k)}$ is a
  supersolution then,
  \begin{gather}
  \label{eq:ccmclaim1}
  z^{(k+1)} \le z^{(k)}\ ,\\
  \label{eq:ccmclaim2}
  z^{(k+1)} \text{ is a supersolution}\ .
  \end{gather}
  Then the lemma follows by induction on $k$. Let us assume that
  $z^{(k)}$ is a supersolution and try to prove~\eqref{eq:ccmclaim1}
  and~\eqref{eq:ccmclaim2}. To prove these we will show that
  $z^{(k,j)} \le z^{(k)}$ and $z^{(k,j)}$ is a supersolution by
  induction on $j \in \cbr{0,1,\hdots,d}$. This proves
  \eqref{eq:ccmclaim1} and \eqref{eq:ccmclaim2} for $z^{(k+1)}$ since
  $z^{(k+1)} = z^{(k,d)}$. .

  The base case ($j=0$) of the induction is trivial since $z^{(k,0)}
  \le z^{(k)}$ since the two vectors are equal. For the same reason,
  $z^{(k,0)}$ is a supersolution since we have assumed $z^{(k)}$ to be
  a supersolution.  Now assume $z^{(k,j-1)} \le z^{(k)}$ and
  $z^{(k,j-1)}$ is a supersolution for some $j > 0$. We want to show
  that $z^{(k,j)} \le z^{(k)}$ and $z^{(k,j)}$ is a supersolution. If
  the update to $z^{(k,j)}$ was trivial, i.e. $z^{(k,j-1)} =
  z^{(k,j)}$ then there is nothing to prove. Therefore, for the
  remainder of the proof assume that the update is non-trivial (and
  hence Lemma~\ref{lem:CCM_update} applies).

  Since $z^{(k,j-1)}$ and $z^{(k,j)}$ differ only in the $j$th
  coordinate, to show that $z^{(k,j)} \le z^{(k)}$ given that
  $z^{(k,j-1)} \le z^{(k)}$, it suffices to show that $z^{(k,j)} \le z^{(k,j-1)}$, i.e.
  \begin{equation}
  \label{eq:ccmjthentry}
	z^{(k,j)}_j \le z^{(k,j-1)}_j = z^{(k)}_j \ .
  \end{equation}
  As in Lemma~\eqref{lem:CCM_update}, let us denote $\fres{j}(\alpha;
  z^{(k,j-1)}$ by $g(\alpha)$. The lemma gives us a $\tau \in
  (0,L]$ such that,
  \begin{equation}
  \label{eq:fromrep}
	z^{(k,j)}_j = S_{\lambda/\tau} \left( z^{(k,j-1)}_j - \frac{[
            \grad f(z^{(k,j-1)}) ]_j}{\tau} \right)\ .
  \end{equation}
  Since $z^{(k,j-1)}$ is a supersolution by induction hypothesis and
  $\tau \leq L$, using Lemma~\ref{lem:monotonic} we get
  \begin{align*}
  z^{(k,j)}_j &\le S_{\lambda/L} \left( z^{(k,j-1)}_j - \frac{[ \grad
      f(z^{(k,j-1)}) ]_j}{L} \right) \le S_{\lambda/L} \left(
  z^{(k)}_j - \frac{[ \grad f(z^{(k)}) ]_j}{L} \right) \le
  z^{(k)}_j \ .
  \end{align*}
  where the second inequality above holds because $z^{(k,j-1)} \le
  z^{(k)}$ by induction hypothesis and since $\Sbb_{\lambda/L} \circ
  (\Ib - \grad f/L)$ is an isotone operator. The third holds since
  $z^{(k)}$ is a supersolution (coupled with
  Lemma~\ref{lem:scale}). Thus, we have proved~\eqref{eq:ccmjthentry}.

  We now need to prove that $z^{(k,j)}$ is a
  supersolution. To this end, we first claim that
  \begin{equation}
  \label{eq:equaldiff}
	z^{(k,j-1)}_j - \frac{[\grad f(z^{(k,j-1)})]_j}{\tau} =
        z^{(k,j)}_j - \frac{[\grad f(z^{(k,j)})]_j}{\tau} \ .
  \end{equation}
  This is true since
  \begin{align*}
  &\quad z^{(k,j-1)}_j - \frac{[\grad f(z^{(k,j-1)})]_j}{\tau} -
    z^{(k,j)}_j + \frac{[\grad f(z^{(k,j)})]_j}{\tau} \\ &=
    z^{(k,j-1)}_j - z^{(k,j)}_j - \frac{1}{\tau}( g'(z^{(k,j-1)}_j) -
    g'(z^{(k,j)}_j) ) \\ &= z^{(k,j-1)}_j - z^{(k,j)}_j - (
    z^{(k,j-1)}_j - z^{(k,j)}_j ) = 0\ .
  \end{align*}
  The first equality is true by definition of $g$ and the second
  by~\eqref{eq:diffratio}. Now, applying $S_{\lambda/\tau}$ to both
  sides of~\eqref{eq:equaldiff} and using~\eqref{eq:fromrep}, we get
  \begin{align}
  \notag z^{(k,j)}_j &= S_{\lambda/\tau} \left( z^{(k,j-1)}_j -
  \frac{[ \grad f(z^{(k,j-1)}) ]_j}{\tau} \right) \\
  \label{eq:ccmsuperforj}
  &= S_{\lambda/\tau} \left( z^{(k,j)}_j - \frac{[ \grad f(z^{(k,j)})
    ]_j}{\tau} \right) \ .
  \end{align}
  For $i \neq j$, $z^{(k,j)}_i = z^{(k,j-1)}_i$ and thus we have
 \begin{align*}
    &\quad z^{(k,j-1)}_i - \frac{ [\grad f(z^{(k,j-1)})]_i }{\tau} -
    z^{(k,j)}_i + \frac{ [\grad f(z^{(k,j)})]_j }{\tau} \\ 
    &= -\frac{1}{\tau} \sbr{ [\grad f(z^{(k,j-1)})]_i - [\grad
        f(z^{(k,j)})]_i } \geq 0
  \end{align*}
  The last inequality holds because we have already shown that
  $z^{(k,j-1)} \ge z^{(k,j)}$ and thus by isotonicity of $\Ib - \grad
  f/L$, we have
  \[
	[\grad f(z^{(k,j-1)})]_i - [\grad f(z^{(k,j)})]_i \le
        L(z^{(k,j-1)}_i - z^{(k,j)}_i) = 0\ .
  \]
  Using the monotonic scalar shrinkage operator we have 
  \begin{align*}
    S_{\lambda/\tau}\rbr{z^{(k,j-1)}_i - \frac{ [\grad
          f(z^{(k,j-1)})]_i }{\tau} } \geq
    S_{\lambda/\tau}\rbr{z^{(k,j)}_i - \frac{ [\grad f(z^{(k,j)})]_i
      }{\tau} }
  \end{align*}
  which, using the inductive hypothesis that $z^{(k,j-1)}$ is a
  supersolution, further yields
  \begin{align}
    \label{eq:ccmsuperforothers}
     z^{(k,j)}_i = z^{(k,j-1)}_i 
    \geq S_{\lambda/\tau}\rbr{z^{(k,j-1)}_i - \frac{ [\grad
          f(z^{(k,j-1)})]_i }{\tau} } 
    &\geq S_{\lambda/\tau}\rbr{z^{(k,j)}_i - \frac{ [\grad
          f(z^{(k,j)})]_i }{\tau} } \ .
  \end{align}
  Combining~\eqref{eq:ccmsuperforj} and~\eqref{eq:ccmsuperforothers},
  we get
  \begin{align*}
    z^{(k,j)} \geq \Sbb_{\lambda/\tau}\left(z^{(k,j)} - \frac{\grad
      f(z^{(k,j)})}{\tau} \right)
  \end{align*}
  which proves, using Lemma~\ref{lem:scale}, that $z^{(k,j)}$ is a
  supersolution.
\end{proof}

\subsection{Comparison: \CCD\ vs. \CCM}
\label{subsec:Compare_CCD_CCM}

\begin{theorem}
\label{thm:CCD_CCM_Compare}
Suppose $\cbr{ y^{(k)} }$ and $\cbr{ z^{(k)} }$ are the sequences of
iterates generated by the \CCD\ and \CCM\ algorithms respectively when
started from the same supersolution $y^{(0)} = z^{(0)}$. Then,
$\forall k\ge 0$,
\[
	z^{(k)} \le y^{(k)}\ .
\]
On the other hand, if they are started from the same subsolution
$y^{(0)} = z^{(0)}$ then the sequences satisfy, $\forall k \ge 0$,
\[
	z^{(k)} \ge y^{(k)}\ .
\]
\end{theorem}
\begin{proof}
  We will only prove the supersolution case as the subsolution case is
  analogous. Given that $y^{(0)} = z^{(0)}$ is a supersolution, we
  will prove the following: if $z^{(k)} \le y^{(k)}$ then,
  \begin{equation}
  \label{eq:ccdccmclaim}
  z^{(k+1)} \le y^{(k+1)} \ .
  \end{equation}
  Then the lemma follows by induction on $k$. Let us assume $z^{(k)}
  \le y^{(k)}$ and try to prove~\eqref{eq:ccdccmclaim}.  To this end
  we will show that $z^{(k,j)} \le y^{(k,j)}$ by induction on $j \in
  \cbr{0,1,\hdots,d}$. This infers \eqref{eq:ccdccmclaim} since
  $z^{(k+1)} = z^{(k,d)}$ and $y^{(k+1)} = y^{(k,d)}$.

  The base case ($j=0$) is true by the given condition in the lemma
  since $z^{(k,0)} = z^{(k)}$ as well as $y^{(k,0)} = y^{(k)}$. Now,
  assume $z^{(k,j-1)} \le y^{(k,j-1)}$ for some $j > 0$. We want to
  show that $z^{(k,j)} \le y^{(k,j)}$.

  Since $z^{(k,j-1)}, z^{(k,j)}$ and $y^{(k,j-1)}, y^{(k,j)}$ differ
  only in the $j$th coordinate, to show that $z^{(k,j)} \le y^{(k,j)}$
  given that $z^{(k,j-1)} \le y^{(k,j-1)}$, it suffices to show that
  \begin{equation}
  \label{eq:ccdccmjthentry}
  z^{(k,j)}_j \le y^{(k,j)}_j\ .
  \end{equation}
  If the update to $z^{(k,j)}$ is non-trivial then using
  Lemma~\ref{lem:CCM_update}, there is a $\tau \in (0,L]$, such that
  \begin{align}
  \notag 
  z^{(k,j)}_j &= S_{\lambda/\tau}\rbr{z^{(k,j-1)}_j - \frac{
      [\grad f(z^{(k,j-1)})]_j }{\tau} } \\
  \label{eq:zkjupper}
  &\leq S_{\lambda/L}\rbr{z^{(k,j-1)}_j - \frac{ [\grad
        f(z^{(k,j-1)})]_j }{L} } \ ,
  \end{align}
  where the last inequality holds because of Lemma~\ref{lem:monotonic}
  and the fact that $z^{(k,j-1)}$ is a supersolution
  (Lemma~\ref{lem:CCM_Compare}).  If the update is trivial,  
  i.e. $z^{(k,j)}_j = z^{(k,j-1)}_j$ then 
  using \eqref{eq:ccmup1} and \eqref{eq:optimality_cond} we have 
  \[
  0 \in [\grad f(z^{(k,j)})]_j + \lambda\sgn(z^{(k,j)}_j)\ .
  \]
  which coupled with \eqref{eq:equivalence} gives 
  \[ 
  z^{(k,j)}_j = S_{\lambda/L}\rbr{z^{(k,j)}_j - \frac{[\grad
        f(z^{(k,j)})]_j}{L}} \leq S_{\lambda/L}\rbr{z^{(k,j-1)}_j -
    \frac{[\grad f(z^{(k,j-1)})]_j}{L}}
  \]
  where the last inequality is obtained by applying the isotone
  operator $\Sbb_{\lambda/L} \circ (\Ib - \grad f/L)$ to the
  inequality $z^{(k,j)}\leq z^{(k,j-1)}$ which holds by lemma
  \ref{lem:CCM_Compare}.  Thus \eqref{eq:zkjupper} holds irrespective
  of the triviality of the update. 

  Now applying the same isotone operator
  to the inequality
  $z^{(k,j-1)} \le y^{(k,j-1)}$ and taking the $j$th coordinate gives,
  \[
  	S_{\lambda/L}\rbr{z^{(k,j-1)}_j - \frac{ [\grad
              f(z^{(k,j-1)})]_j }{L} } \le
        S_{\lambda/L}\rbr{y^{(k,j-1)}_j - \frac{ [\grad
              f(y^{(k,j-1)})]_j}{L} }\ .
  \]
  The right hand side above is, by definition, $y^{(k,j)}_j$. So,
  combining the above with~\eqref{eq:zkjupper}
  gives~\eqref{eq:ccdccmjthentry} and proves our inductive claim.
\end{proof}

\section{Convergence Rates}
\label{sec:Rates}

Our results so far have given inequalities comparing the iterates
generated by the three algorithms. We finally want to compare the
function values obtained by these iterates. For doing that, the next
lemma is useful.

\begin{lemma}
  \label{lem:func_compare}
  If $y$ is a supersolution and $y \le x$ then $F(y) \le F(x)$.
\end{lemma}
\begin{proof}
  Since $F$ is convex, we have 
  \begin{align}
    \label{eq:func_value_inequality}
    F(y) - F(x) &\leq \inner{\grad f(y) + \lambda \rho}{y - x}
  \end{align}
  for any $\rho \in \partial\|y\|_1$.
  We have assumed that $y \leq x$. Thus in order to prove 
  $F(y) - F(x) \le 0$, it suffices to show that 
  \begin{align}
    \label{eq:func_values_condition}
    \forall i\in[d], \qquad \exists \rho_i \in \sgn(y_i) \qquad
    \text{s.t.} \qquad
    \gamma_i + \lambda \rho_i \ge 0
  \end{align}
  where, for convenience, we denote the gradient $\grad f(y)$ by $\gamma$.
  Since $y$ is a supersolution, Lemma~\ref{lem:scale} gives, 
  \begin{align}
    \label{eq:Super_sol_relation}
    \forall i\in[d], \qquad y_i \geq S_{\lambda/L}\left(y_i -
    \frac{\gamma_i}{L}\right)
  \end{align}
  For any $i \in [d]$, there are three mutually exclusive and
  exhaustive cases.
  \begin{description}
    \item[Case (1)]:  $y_i > \frac{\gamma_i + \lambda}{L}$ 
      Plugging this value in \eqref{eq:Super_sol_relation} and using
      the definition of scalar shrinkage~\eqref{eq:scshrinkage}, we
      get
      \begin{align*}
        y_i \geq y_i - \frac{\gamma_i + \lambda}{L}
      \end{align*}
      which gives $\gamma_i + \lambda \geq 0$ and hence $y_i>0$. Thus,
      we can choose $\rho_i = 1 \in \sgn(y_i)$ and we indeed have $\gamma_i
      + \lambda \rho_i \ge 0$.
    \item[Case (2)]: $y_i \in
      [\frac{\gamma_i-\lambda}{L},\frac{\gamma_i + \lambda}{L}]$
      In this case, we have $y_i \geq S_{\lambda/L}(y^{(k)}_i -
      \frac{\gamma_i}{L}) = 0$. Thus,
      \begin{align*}
        \frac{\gamma_i+\lambda}{L} \geq y_i \geq 0\ .
      \end{align*}
      Thus we can choose $\rho_i = 1 \in \sgn(y_i)$ and we have
      $\gamma_i + \lambda \rho_i \ge 0$.
      
    \item[Case (3)]: $y_i < \frac{\gamma_i - \lambda}{L}$ 
      Plugging this
      value in \eqref{eq:Super_sol_relation} and using the
      definition of scalar shrinkage~\eqref{eq:scshrinkage}, we get
      \begin{align*}
        y_i \geq y_i - \frac{\gamma_i - \lambda}{L}
      \end{align*}
        which gives $\gamma_i - \lambda\geq 0$. 
        Now if $y_i \leq 0$, we can set $\rho = -1 \in \sgn(y_i)$ and
        will have $\gamma_i + \lambda \rho_i \geq 0$.  On the other
        hand, if $y_i > 0$, we need to choose $\rho_i = 1$ and thus
        $\gamma_i + \lambda\geq 0$ should hold if
        \eqref{eq:func_values_condition} is to be true. However, we
        know $\gamma_i - \lambda\geq 0$, and $\lambda \geq 0$ so
        $\gamma_i + \lambda\geq 0$ is also true.
  \end{description}
  Thus in all three cases we have that there is a $\rho_i \in
  \sgn(y_i)$ such that~\eqref{eq:func_values_condition} is true.
\end{proof}

There is a similar lemma for subsolutions whose proof, being
similar to the proof above, is skipped.
\begin{lemma}
  If $y$ is a subsolution and $y \ge x$ then $F(y) \le F(x)$.
\end{lemma}

If we start from a supersolution, the iterates for \CCD\ and
\CCM\ always maintain the supersolution property. Thus Lemma
\ref{lem:func_compare} ensures that starting from the same initial
iterate, the function values of the \CCD\ and \CCM\ iterates always
remain less than the corresponding \GD\ iterates. Since the
\GD\ algorithm has $O(1/k)$ accuracy guarantees according to Theorem
\ref{thm:gd}, the same rates must hold true for \CCD\ and \CCM. This
is formalized in the following theorem.
\begin{theorem}   
  \label{thm:main_theorem}
  Starting from the same super- or subsolution $x^{(0)} = y^{(0)} = z^{(0)}$,
  let $\cbr{x^{(k)}}$, $\cbr{y^{(k)}}$ and $\cbr{z^{(k)}}$ denote the
  \GD, \CCD\ and \CCM\ iterates respectively. Then for any minimizer
  $x^*$ of \eqref{eq:Reg_l_1_loss}, and $\forall k\geq 1$, 
  \[
  F(z^{(k)}) \le F(y^{(k)}) \le F(x^{(k)})
  \le F(x^\star) + \frac{ L \| x^\star - x^{(0)} \|^2}{2\,k}
  \] 
\end{theorem}
\section{Conclusion}
\label{sec:Conclusion}
Coordinate descent based methods have seen a resurgence of popularity in
recent times in both the machine learning and the statistics
community, due to the simplicity of the updates and implementation of
the overall algorithms. Absence of finite time
convergence rates is thus one of the most important theoretical issues to 
address.  

In this paper, we provided a comparative analysis of \GD, \CCD\ and
\CCM\ algorithms to give the first known finite time guarantees on the
convergence rates of cyclic coordinate descent methods. However, there
still are a significant number of unresolved questions. Our
comparative results require that the algorithms start from a
supersolution so that the property is maintained for all the
subsequent iterates. We also require 
an isotonicity assumption on the $\Ib - \grad f/L$ operator. Although this is a 
fairly common assumption in numerical optimization \citep{BertTsit89}, it is
desirable to have a more generalized analysis without any restrictions.  
Since stochastic coordinate descent \citep{ShaiAmbuj09}
converges at the same $O(1/k)$ rate as \GD\ without additional assumptions,
intuition suggests that same should be true for \CCD\ and \CCM.
A theoretical proof of the same remains an open question. 

Some greedy versions of the coordinate descent algorithm (e.g., \citep{WuLange08}) still lack
a theoretical analysis of their finite time convergence guarantees. Although \cite{Clarkson08} 
has a $O(1/k)$ rates for a greedy version, the analysis is restricted 
to a simplex domain and does not generalize to arbitrary domains. The phenomenal performance 
of greedy coordinate descent algorithms on real life datasets 
makes it all the more essential to validate these experimental results theoretically.

\bibliographystyle{alpha} 
\bibliography{Paper}

\newpage
\appendix
\section*{Appendix}
\label{sec:Appendix}

\section{Proof of Lemma \ref{lem:CCM_update}}
\label{sec:CCM_Update}
Since $g(\alpha) = \fres{i}(\alpha; z^{(k,j-1)})$ we have
  \begin{align*}
    g'(\alpha) = \left[\grad f(z^{(k,j-1)}_1,z^{(k,j-1)}_2,\hdots
      z^{(k,j-1)}_{j-1},\alpha, z^{(k,j-1)}_{j+1},\hdots
      z^{(k,j-1)}_d)\right]_j
  \end{align*}
  Therefore, 
  \begin{align}
    \label{eq:equiv_f_g}
    g'(z^{(k,j-1)}_j) = [\grad f(z^{(k,j-1)})]_j 
  \end{align}
  Since, by definition, $z^{(k,j)}_j$ is the minimizer of $g(\alpha) +
  \lambda|\alpha|$, we have
  \begin{align*}
    0 \in g'(z^{(k,j)}_j) + \lambda \sgn(z^{(k,j)}_j)
  \end{align*}
  For notational convenience we denote $z^{(k,j)}_j$ as
  $\alpha^\star$, since it is the minimizer of $g(\alpha) +
  \lambda|\alpha|$.  With this notation we have,
  \begin{equation}
  \label{eq:betadef}
  	\tau = \frac{ g'(\alpha^\star) - g'(z^{(k,j-1)}_j)
        }{\alpha^\star - z^{(k,j-1)}_j}\ .
  \end{equation}
  Note that $\tau$ is well defined since the denominator is non-zero
  by our assumption of a non-trivial update. Further, $\tau > 0$ by
  Assumption~\ref{asmp:strict} and $\tau \leq L$ since $\grad f$ (and
  hence $g'(\alpha)$) is $L$-Lipschitz continuous.

  Depending on the sign of $\alpha^\star$, there are three possible
  cases:
  \begin{description}
    \item[Case (1): $\alpha^\star >0$:] This implies that
      \begin{align}
        \label{eq:case1_alpha}
        g'(\alpha^\star)+\lambda = 0
      \end{align}
      By~\eqref{eq:betadef}, 
      \begin{align*}
        g'(\alpha^\star) = g'(z^{(k,j-1)}_j) + \tau (\alpha^\star -
        z^{(k,j-1)}_j)
      \end{align*}
      Plugging this in \eqref{eq:case1_alpha}, we get 
      \begin{align*}
        g'(z^{(k,j-1)}_j) + \tau(\alpha^\star - z^{(k,j-1)}_j) + \lambda = 0 \ .
      \end{align*}
      Using the definition of shrinkage
      operator~\eqref{eq:scshrinkage} combined with the fact that
      $\alpha^\star>0$, we have
      \begin{align*}
        \alpha^\star &= z^{(k,j-1)}_j -
        \frac{1}{\tau}g'(z^{(k,j-1)}_j) - \frac{\lambda}{\tau} \\ &=
        S_{\lambda/\tau}\rbr{z^{(k,j-1)}_j -
          \frac{g'(z^{(k)}_j)}{\tau} }
      \end{align*}
      \item[Case (2): $\alpha^\star = 0$:] The corresponding condition is 
        \begin{align*}
          0 \in [g'(\alpha^\star)-\lambda, g'(\alpha^\star)+\lambda]
        \end{align*}
        Again using~\eqref{eq:betadef}, we have 
        \begin{align*}
          g'(\alpha^\star) &= g'(z^{(k,j-1)}_j) + \tau(\alpha^\star -
          z^{(k,j-1)}_j) = g'(z^{(k,j-1)}_j) - \tau( z^{(k,j-1)}_j)
          \qquad \text{[since $\alpha^\star = 0$]} \\ \implies
          \alpha^\star &= 0 \in \left[\frac{g'(z^{(k,j-1)}_j)}{\tau} -
            z^{(k,j-1)}_j - \frac{\lambda}{\tau},
            \frac{g'(z^{(k,j-1)}_j)}{\tau} - z^{(k,j-1)}_j +
            \frac{\lambda}{\tau} \right] \\ \implies \alpha^\star &= 0
          = S_{\lambda/\tau}\rbr{z^{(k,j-1)}_j -
            \frac{g'(z^{(k,j-1)}_j)}{\tau}}
        \end{align*}
        where the last step follows from the definition of the
        shrinkage operator~\eqref{eq:scshrinkage}.
      \item[Case (3): $\alpha^\star <0$:] This implies that 
        \begin{align*}
          g'(\alpha^\star)-\lambda = 0
        \end{align*}
        Using~\eqref{eq:betadef} to substitute for $g'(\alpha^\star)$
        as in the previous cases, we have,
        \begin{align*}
          g'(z^{(k,j-1)}_j) + \tau(\alpha^\star - z^{(k,j-1)}_j) -
          \lambda = 0
        \end{align*}
      which yields 
      \begin{align*}
        \alpha^\star &= z^{(k,j-1)}_j -
        \frac{1}{\tau}g'(z^{(k,j-1)}_j) + \frac{\lambda}{\tau} \\ &=
        S_{\lambda/\tau}\rbr{z^{(k,j-1)}_j -
          \frac{g'(z^{(k,j-1)}_j)}{\tau} }
      \end{align*}
      where the last inequality follows because $\alpha^\star < 0$.

      Combining these three cases and using \eqref{eq:equiv_f_g} we get 
      \begin{align*}
        z^{(k,j)}_j = S_{\lambda/\tau}\rbr{z^{(k,j-1)}_j - \frac{
            \sbr{\grad f(z^{(k,j-1)})}_j }{\tau} } \ .
      \end{align*}
  \end{description}

\begin{figure}
\begin{center}
\begin{tikzpicture}[scale=1]
{\footnotesize 
\draw[<->] (-2,0) -- (7,0); 
\draw[<->] (0,-2) -- (0,2); 

\draw[dashed] (-2,-2) -- (2,2) node[right] {$y=x_j$} ;

\draw[very thick] (-1,-2) -- (1,0); 
\draw[very thick] (1,0) -- (5,0); 
\draw[very thick] (5,0) -- (7,2) node[right]{$y = S_{\lambda}\rbr{x_j - [\grad f(x)]_j}$}; 

\draw[very thick] (3,-.1) -- (3,.1); 

\node at (1,.5)  {$[\nabla f(x)]_j-\lambda$}; 
\node at (3,-.5) {$[\nabla f(x)]_j$}; 
\node at (5,-.5) {$[\nabla f(x)]_j+\lambda$}; 
}

\end{tikzpicture}
\end{center}
\caption{Interval to right of zero}
\label{fig:fig1}
\end{figure}
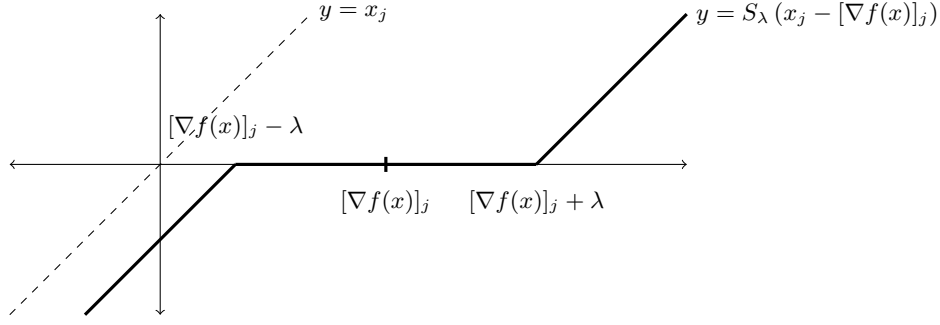

\begin{figure}
\begin{center}
\begin{tikzpicture}[scale=1]
{\footnotesize 
\draw[<->] (-3,0) -- (5,0); 
\draw[<->] (0,-2) -- (0,2); 

\draw[dashed] (-2,-2) -- (2,2) node[right] {$y=x_j$};

\draw[very thick] (-3,-2) -- (-1,0); 
\draw[very thick] (-1,0) -- (3,0); 
\draw[very thick] (3,0) -- (5,2) node[right] {$y = S_{\lambda}\rbr{x_j - [\grad f(x)]_j}$}; 

\draw[very thick] (1,-.1) -- (1,.1); 

\node at (-1,.5)  {$[\nabla f(x)]_j-\lambda$}; 
\node at (1,-.5) {$[\nabla f(x)]_j$}; 
\node at (3,-.5) {$[\nabla f(x)]_j+\lambda$}; 
}

\end{tikzpicture}
\end{center}
\caption{Interval crossing zero}
\label{fig:fig2}
\end{figure}
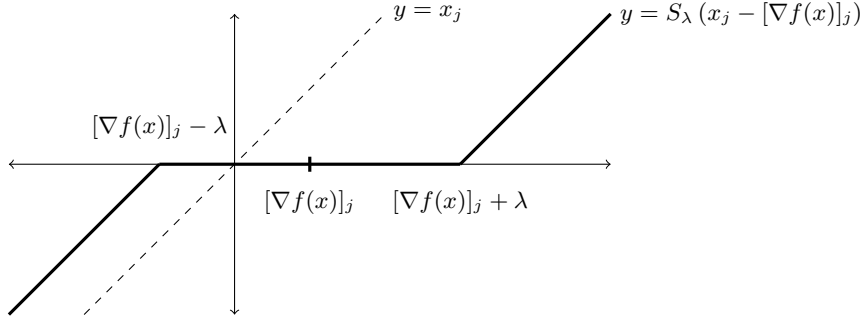

\section{Proof of lemma \ref{lem:scale}}
\label{sec:scale_proof}
  We prove the supersolution case only as the subsolution case is analogous.
  Let for a particular $\tau >0$, $x \geq \Sbb_{\lambda/\tau}\rbr{x -
    \frac{\grad f(x)}{\tau}}$.  We prove the inequality for the scalar
  $S$ operator on an arbitrary coordinate $j$. The subsequent proofs are divided 
  into three disjoint cases related to the values taken by the shrinkage operator. 
  \begin{description}
    \item \textbf{Case 1 $[\grad f(x)]_j - \lambda > 0$:} 

      This is illustrated in figure \ref{fig:fig1}. Depending on
      whether $\tau > 1$ or not, the graph of the shrinkage operator
      shifts left or right, but clearly division by $\tau$ does not
      change the sign of the shrinkage operator value at any
      point. As is evident from figure \ref{fig:fig1}, the graph of $y= x_j$ always lies 
      above that of the shrinkage operator. Thus
      \begin{align}
        \label{eq:sup_sol_inequality}
        x_j \geq S_{\lambda/\tau}\rbr{x_j - \frac{[\grad f(x)]_j}{\tau}}
      \end{align}
        for all values of $\tau$ and in particular for $\tau =
        1$. Thus $x$ is a supersolution.
    
    \item \textbf{Case 2 $0 \in [ [\grad f(x)]_j - \lambda, [\grad f
          (x)]_j + \lambda]$:}
      
      The corresponding case is illustrated in figure
      \ref{fig:fig2}. It is clear from the figure that $x_j \geq
      S_{\lambda/\tau}\rbr{x_j - \frac{[\grad f(x)]_j}{\tau}}$ for
      positive $\tau$, only when $x_j \geq 0 $. Just as in the previous
      case, changing the value of $\tau$ shifts the graph by
      appropriate scale without changing its sign. Thus \eqref{eq:sup_sol_inequality} holds 
      for $x_j \geq 0$ irrespective of the value of $\tau$. In particular, it
      should hold for $\tau = 1$ which proves that $x$ is a supersolution. 

      \item \textbf{Case 3 $[\grad f(x)]_j + \lambda < 0$:} 

        As illustrated in figure \ref{fig:fig3}, in this case the
        graph of the shrinkage operator will always lie below the
        value of $x_j$. Thus \eqref{eq:sup_sol_inequality} will not be
        satisfied for any value of $\tau$ which makes the case
        vacuous.
  \end{description}

  To prove the converse direction, we look at the same three exclusively
  disjoint cases for an arbitrary coordinate $j$. 
\begin{description}
  \item \textbf{Case 1 $[\grad f(x)]_j - \lambda > 0$:}

    As seen from figure \ref{fig:fig1}, $x$ is always a supersolution
    since $[\grad f(x)]_j + \lambda > [\grad f(x)]_j - \lambda > 0$
    and the graph of the shrinkage operator uniformly stays below the
    value of $x_j$. Since the sign of the shrinkage operator value
    does not change due to division by $\tau >0$,
    \eqref{eq:sup_sol_inequality} holds for arbitrary positive $\tau$.
    
  \item \textbf{Case 2 $0 \in [ [\grad f(x)]_j - \lambda, [\grad f
        (x)]_j + \lambda]$:}

      If $x$ is a supersolution, it means that the value attained by the
      shrinkage operator lies below the value of $x_j$ , which
      is true when $x_j \geq 0$ (Figure \ref{fig:fig2}). In this subset of
      the domain, division by $\tau$ maintains the sign of the shrinkage value
      and thus \eqref{eq:sup_sol_inequality} holds.
    
    \item \textbf{Case 3 $[\grad f(x)]_j + \lambda < 0$:} 

      In this case the graph of the shrinkage operator always lies
      above the value of $x_j$ . Thus $x$ can never be a supersolution
      if this condition holds true.

\end{description}

\begin{figure}
\begin{center}
\begin{tikzpicture}[scale=1]
{\footnotesize 
\draw[<->] (-7,0) -- (1,0); 
\draw[<->] (0,-2) -- (0,2); 

\draw[dashed] (-2,-2) -- (2,2) node[right] {$y=x_j$};

\draw[very thick] (-7,-2) -- (-5,0); 
\draw[very thick] (-5,0) -- (-1,0); 
\draw[very thick] (-1,0) -- (1,2) node[left] {$y = S_{\lambda}\rbr{x_j - [\grad f(x)]_j}$}; 

\draw[very thick] (-3,-.1) -- (-3,.1); 

\node at (-5,.5)  {$[\nabla f(x)]_j-\lambda$}; 
\node at (-3,.5) {$[\nabla f(x)]_j$}; 
\node at (-1,-.5) {$[\nabla f(x)]_j+\lambda$}; 
}

\end{tikzpicture}
\end{center}
\caption{Interval to left of zero}
\label{fig:fig3}
\end{figure}
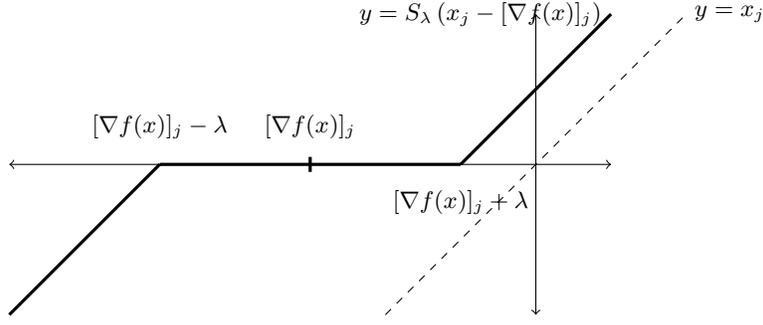

\section{Proof of lemma \ref{lem:monotonic}}
\label{sec:monotonic_proof}
Let 
\[
h(\tau) = S_{\lambda/\tau}\rbr{x_j - \frac{[\grad f(x)]_j}{\tau}}
\] 
We again look at the three disjoint cases for arbitrary $\tau_1,
\tau_2 \in (0, \infty)$ with $\tau_1\geq\tau_2$ and show that
$h(\tau_1)\geq h(\tau_2)$.
\begin{description}
  \item \textbf{Case 1 $[\grad f(x)]_j - \lambda > 0$:} 

    Since both the hinge points in the graph will be positive (figure \ref{fig:fig4}
    ), we have $\frac{[\grad f(x)]_j-\lambda}{\tau_1}\leq
    \frac{[\grad f(x)]_j-\lambda}{\tau_2}$ and $\frac{[\grad
        f(x)]_j+\lambda}{\tau_1}\leq \frac{[\grad
        f(x)]_j+\lambda}{\tau_2}$. 
    Thus it is trivial to see that the graph of $h(\tau_1)$ is always greater than $h(\tau_2)$.
 
  \item \textbf{Case 2 $0 \in [[\grad f(x)]_j - \lambda, [\grad
          f(x)]_j + \lambda]$:}

      Since $x$ needs to be a supersolution, we only need to consider
      the subset of the domain when $x_j \geq 0$. We still have
      $\frac{[\grad f(x)]_j+\lambda}{\tau_1}\leq \frac{[\grad
          f(x)]_j+\lambda}{\tau_2}$ and it is obvious from figure \ref{fig:fig5}, 
      that $h(\tau_1) \geq h(\tau_2)$. 

    \item \textbf{Case3 $[\grad f(x)]_j + \lambda < 0:$} 
      
      Since $x$ can never be a supersolution in this case as shown in
      the proof of lemma \ref{lem:scale}, this case is vacuous.
\end{description}

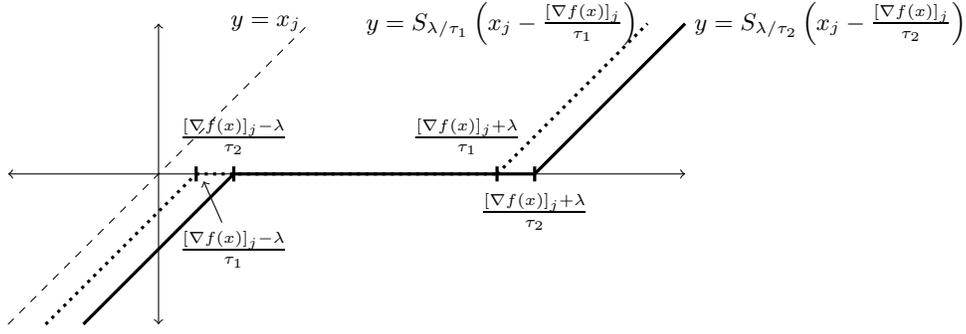
\begin{figure}
\begin{center}
\begin{tikzpicture}[scale=1]
{\footnotesize 
\draw[<->] (-2,0) -- (7,0); 
\draw[<->] (0,-2) -- (0,2); 

\draw[dashed] (-2,-2) -- (2,2) node[left] {$y=x_j$};

\draw[very thick] (-1,-2) -- (1,0); 
\draw[very thick] (1,0) -- (5,0); 
\draw[very thick] (5,0) -- (7,2) node[right]{$y = S_{\lambda/\tau_2}\rbr{x_j - \frac{[\grad f(x)]_j}{\tau_2}}$}; 

\draw[very thick,dotted] (-1.5,-2) -- (0.5,0); 
\draw[very thick,dotted] (0.5,0) -- (4.5,0); 
\draw[very thick,dotted] (4.5,0) -- (6.5,2) node[left]{$y = S_{\lambda/\tau_1}\rbr{x_j - \frac{[\grad f(x)]_j}{\tau_1}}$}; 

\draw[very thick] (0.5,-.1) -- (0.5,.1); 
\draw[very thick] (1,-.1) -- (1,.1); 
\draw[very thick] (4.5,-.1) -- (4.5,.1); 
\draw[very thick] (5,-.1) -- (5,.1); 

\node at (1,.5)  {$\frac{[\nabla f(x)]_j-\lambda}{\tau_2}$}; 
\node at (5,-.5) {$\frac{[\nabla f(x)]_j+\lambda}{\tau_2}$}; 

\node (tnode) at (1,-1)  {$\frac{[\nabla f(x)]_j-\lambda}{\tau_1}$}; 
\node at (4.1,.5) {$\frac{[\nabla f(x)]_j+\lambda}{\tau_1}$}; 

\draw [->](tnode) -- (0.6,-.1);  
}

\end{tikzpicture}
\end{center}
\caption{Interval to right of zero}
\label{fig:fig4}
\end{figure}

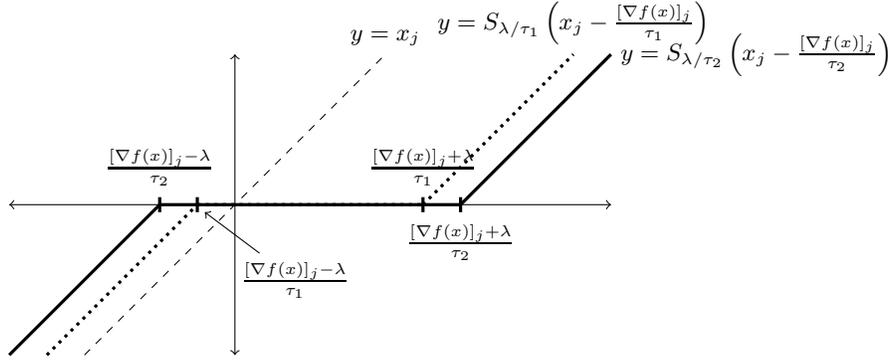
\begin{figure}
\begin{center}
\begin{tikzpicture}[scale=1]
{\footnotesize 
\draw[<->] (-3,0) -- (5,0); 
\draw[<->] (0,-2) -- (0,2); 

\draw[dashed] (-2,-2) -- (2,2)node[above]{$y=x_j$};

\draw[very thick] (-3,-2) -- (-1,0); 
\draw[very thick] (-1,0) -- (3,0); 
\draw[very thick] (3,0) -- (5,2) node[right]{$y = S_{\lambda/\tau_2}\rbr{x_j - \frac{[\grad f(x)]_j}{\tau_2}}$}; 

\draw[very thick,dotted] (-2.5,-2) -- (-0.5,0); 
\draw[very thick,dotted] (-0.5,0) -- (2.5,0); 
\draw[very thick,dotted] (2.5,0) -- (4.5,2) node[above]{$y = S_{\lambda/\tau_1}\rbr{x_j - \frac{[\grad f(x)]_j}{\tau_1}}$}; 

\draw[very thick] (-1,-.1) -- (-1,.1); 
\draw[very thick] (-.5,-.1) -- (-.5,.1); 
\draw[very thick] (2.5,-.1) -- (2.5,.1); 
\draw[very thick] (3,-.1) -- (3,.1); 

\node at (-1,.5)  {$\frac{[\nabla f(x)]_j-\lambda}{\tau_2}$}; 
\node at (3,-.5) {$\frac{[\nabla f(x)]_j+\lambda}{\tau_2}$}; 

\node (tnode) at (0.8,-1)  {$\frac{[\nabla f(x)]_j-\lambda}{\tau_1}$}; 
\node at (2.5,.5) {$\frac{[\nabla f(x)]_j+\lambda}{\tau_1}$}; 

\draw [->] (tnode) -- (-0.4, -.1);
}

\end{tikzpicture}
\end{center}
\caption{Interval crossing zero}
\label{fig:fig5}
\end{figure}

\end{document}